\documentclass{article}
\usepackage{kr}

\usepackage{times}

\usepackage{latexsym}
\usepackage{amssymb}
\usepackage{amsmath}
\usepackage{amsthm}
\usepackage{booktabs}
\usepackage{enumitem}
\usepackage{graphicx}
\usepackage{array}
\usepackage{multirow}
\usepackage{amsfonts}
\usepackage{subcaption}
\usepackage[ruled,vlined,linesnumbered]{algorithm2e}
\usepackage{xspace}
\usepackage{caption}
\usepackage{makecell}
\usepackage[table]{xcolor}

\newcommand{\fluents}{\ensuremath{F}\xspace}
\newcommand{\actions}{\ensuremath{A}\xspace}
\newcommand{\init}{\ensuremath{I}\xspace}
\newcommand{\goal}{\ensuremath{G}\xspace}
\newcommand{\cost}{\ensuremath{c}\xspace}

\newcommand{\task}{\ensuremath{{\cal P}}\xspace}
\newcommand{\plan}{\ensuremath{\pi}\xspace}
\newcommand{\plans}{\ensuremath{\Pi}\xspace}
\newcommand{\stripstask}{\ensuremath{\task=\langle \fluents, \actions, \init, \goal \rangle}\xspace}
\newcommand{\state}{\ensuremath{s}\xspace}

\newcommand{\action}{\ensuremath{a}\xspace}
\newcommand{\name}{\ensuremath{\textsc{name}}\xspace}
\newcommand{\precondition}{\ensuremath{\textsc{pre}}\xspace}

\newcommand{\addeffects}{\ensuremath{\textsc{add}}\xspace}
\newcommand{\deleffects}{\ensuremath{\textsc{del}}\xspace}
\newcommand{\actionapplication}{\ensuremath{\gamma}\xspace}
\newcommand{\planapplication}{\ensuremath{\Gamma}\xspace}
\newcommand{\alternativeplans}[1]{\ensuremath{\plans^{#1}}\xspace}
\newcommand{\alternativeplan}{\ensuremath{\plan^{\mathsf{a}}}\xspace}

\newcommand{\commitfluents}{\ensuremath{F_\mathsf{c}}\xspace}
\newcommand{\commitactions}{\ensuremath{A_\mathsf{c}}\xspace}

\newcommand{\commitgoal}{\ensuremath{G_\mathsf{c}}\xspace}
\newcommand{\goalhat}{\bar{\ensuremath{G}}\xspace}
\newcommand{\subgoal}{\ensuremath{g}\xspace}
\newcommand{\Sticky}{\mathsf{commit}\xspace}
\newcommand{\sticky}{\mbox{-}\mathsf{commit}\xspace}
\newcommand{\forcesticky}{\mbox{-}\mathsf{forcecommit}\xspace}
\newcommand{\Forcesticky}{\mathsf{forcecommit}\xspace}
\newcommand{\actionssticky}{\actions^{\mathsf{C}}\xspace}
\newcommand{\actionsnoncommit}{\actions^{\neg\mathsf{C}}\xspace}
\newcommand{\actionnoncommit}{\action^{\neg\mathsf{C}}\xspace}
\newcommand{\actioncommit}{\action^{\mathsf{C}}\xspace}
\newcommand{\actionsaddgoals}{\actions^\goal\xspace}
\newcommand{\actionsdeletegoals}{\actions^{\neg\goal}\xspace}
\newcommand{\actionslonely}{\actions^\ensuremath{L}\xspace}
\newcommand{\stickytask}{\ensuremath{{\cal P}_{\mathsf{c}}}\xspace}
\newcommand{\actionaddgoals}{\action^\goal\xspace}
\newcommand{\actiondeletegoals}{\action^{\neg\goal}\xspace}
\newcommand{\commitgoals}{\ensuremath{C}_{\actionaddgoals}\xspace}

\newcommand{\actionsdeladdgoals}{\actions^{\goal^{\star}}\xspace}
\newcommand{\actiondeladdgoals}{\action^{\goal^{\star}}\xspace}
\newcommand{\commitdeladdgoals}{\ensuremath{C}_{\actiondeladdgoals}\xspace}
\newcommand{\actionssim}{\actions^{\mathsf{S}}}
\newcommand{\actionsim}{\action^{\mathsf{S}}}

\newcommand{\simultaneous}{\mbox{-}\mathsf{simultaneous}\xspace}

\newcommand{\lamafirst}{\ensuremath{\textsc{lamaF}}\xspace}
\newcommand{\lmcut}{\ensuremath{\textsc{lmcut}}\xspace}

\newtheorem{theorem}{Theorem}

\newtheorem{corollary}[theorem]{Corollary}

\newtheorem{definition}{Definition}

\title{A Planning Compilation to Reason about Goal Achievement at Planning Time}

\author{%
Alberto Pozanco\and
Marianela Morales\and
Daniel Borrajo\and
Manuela Veloso \\
\affiliations
J.P. Morgan AI Research\\
\emails
\{alberto.pozancolancho,marianela.moraleselena\}@jpmorgan.com, \{name.surname\}@jpmorgan.com}

\begin{document}

\maketitle

\begin{abstract}
Identifying the specific actions that achieve goals when solving a planning task might be beneficial for various planning applications.
Traditionally, this identification occurs post-search, as some actions may temporarily achieve goals that are later undone and re-achieved by other actions.
In this paper, we propose a compilation that extends the original planning task with commit actions that enforce the persistence of specific goals once achieved, allowing planners to identify permanent goal achievement during planning.
Experimental results indicate that solving the reformulated tasks does not incur on any additional overhead both when performing optimal and suboptimal planning, while providing useful information for some downstream tasks.
\end{abstract}

\section{Introduction}
Automated Planning involves determining a sequence of actions, or a plan, to achieve a set of goals from an initial state~\cite{DBLP:books/daglib/0014222}.
Identifying the specific actions that achieve goals when solving a planning task might be beneficial for various planning applications.
For example, it can be valuable for attributing goal achievement to agents in centralized multi-agent planning~\cite{pozanco2022fairness}; or for analyzing goal achievement distribution throughout the plan.
However, this identification can only be done at the end of planning by analyzing the returned plan, since some actions may temporarily achieve goals that are later undone and re-achieved by other actions.

Consider the \textsc{sokoban} task illustrated in Figure~\ref{fig:sokoban_task}, where two agents are responsible for pushing stones to their designated goal locations, marked in green. 
One possible plan to accomplish this task involves the orange agent pushing a stone three times to the right, followed by the blue agent pushing the other stone once downward. 
In this plan, the first stone temporarily occupies a goal location before being moved to its final destination. 
Through post-processing, we can easily identify that it is not the initial action of the orange agent, but rather the final push executed by the blue agent, that successfully achieves the specific goal. 
However, there is currently no standard mechanism to determine at planning time, whether a goal proposition that becomes true  will maintain its truth value until the plan's completion.
This limitation hinders the community from developing planning solutions that not only assess goal achievement but also reason about and optimize it. 
In the \textsc{sokoban} example, possessing this capability would enable us to create plans that prioritize specific goal achievement distributions over others.
Examples include preferring more \emph{fair} distributions of goal to agents; or plans that minimize the number of actions between the achievement of consecutive goals.

\begin{figure}
    \centering
\includegraphics[width=0.38\columnwidth]{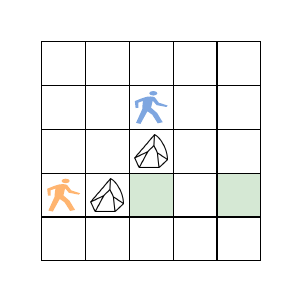}
     \caption{\textsc{sokoban} task where two agents are responsible for pushing stones to their goal locations (in green).} 
     \label{fig:sokoban_task}
    \vspace{-2.0mm}
\end{figure}

In this paper, we introduce a compilation~\cite{nebel2000compilability} that extends the original planning task by incorporating commit actions.
These actions enable the planner to ensure the persistence of specific goals once they are achieved, thus committing to them within the search sub-tree. 
This approach allows the planner to determine, during the planning phase, when an action permanently achieves a goal. 
In the previous \textsc{sokoban} example, the agents would have two options: either perform the standard push or execute a push-$\Sticky$ action, which guarantees that the stone will remain in its goal location for the remainder of the planning episode.

Experimental results across a comprehensive benchmark demonstrate that solving the reformulated tasks incurs in no additional overhead, whether in optimal or suboptimal planning.
This highlights that our compilation enables reasoning about goal achievement during the planning phase without any added overhead.

\section{Background}
We formally define a planning task as follows:
\begin{definition}\label{def:strips-plan-task}
  A {\sc strips} \textbf{planning task}
is a tuple \stripstask, where \fluents is a set of fluents, \actions is a set of
 actions, $\init \subseteq \fluents$ is an initial state, and $\goal\subseteq \fluents$ is a goal specification.  
\end{definition}

A state $\state \subseteq \fluents$ is a set of fluents that are true at a given time.
A state $\state \subseteq \fluents$ is a goal state iff $\goal \subseteq \state$.
Each action $\action \in \actions$ is described by its name $\name(\action)$, a set of positive and negative preconditions $\precondition^+(\action)$ and $\precondition^-(\action)$, add effects $\addeffects(\action)$, delete effects $\deleffects(\action)$, and cost $\cost(\action)$.
An action \action is applicable in a state \state iff $\precondition^+(\action)\subseteq~s$ and $\precondition^-(\action) \cap s = \emptyset$.
We define the result of applying an action in a state as $\actionapplication(\state,\action)=(\state \setminus \deleffects(\action)) \cup \addeffects(\action)$.
We assume $\deleffects(\action) \cap \addeffects(\action) = \emptyset$.
A sequence of actions $\plan=(\action_1,\ldots,\action_n)$ is applicable in a state $\state_0$ if there are states $(\state_1.\ldots,\state_n)$ such that $\action_i$ is applicable in $\state_{i-1}$ and $\state_i=\actionapplication(\state_{i-1},\action_i)$.
The resulting state after applying a sequence of actions is $\planapplication(\state,\pi)=\state_n$, and $\cost(\plan) = \sum_{i}^n \cost(\action_i)$ denotes the cost of $\plan$.
A state $\state$ is reachable from state $\state^\prime$ iff there exists an applicable action sequence \plan such that $s \subseteq \planapplication(\state^\prime,\pi)$.
The solution to a planning task $\task$ is a plan, i.e., a sequence of actions $\plan$ such that $\goal \subseteq \planapplication(\init,\pi)$.
We denote as $\Pi(\task)$ the set of all solution plans to planning task $\task$. 
Also, given a plan \plan, we denote its alternatives, i.e., all the other sequence of actions that can solve \task as $\alternativeplans{\plan} = \Pi(\task) \setminus \pi$.
A plan with minimal cost is optimal.

\section{Planning with Commit Actions}
The aim of our compilation is to extend the original task with $\Sticky$ actions that allow the planner to enforce the persistence of specific goals once they are achieved.
In particular, we will only focus on committing to those goals that are not already true in the initial state, and that can be achieved through actions in $A$.
We make this restriction because it is not possible to tell apriori whether the goals that are true in \init will need to be eventually undone to then be re-achieved by an action, or no action will need to ever falsify the goal in order to achieve $\goal$. We define pending goals as: 
\begin{definition}
    Let \stripstask be a planning task, the subset of \textbf{pending goals} $\goalhat \subseteq \goal$ is formally defined as:\small
    \[
    \goalhat =  \{\subgoal_{i} \in \goal\setminus\init \mid \exists \action \in \actions, \subgoal_i \in \addeffects(\action)\}
    \]
\end{definition}
To ensure that goals in $\goalhat$ remain achieved throughout the planning process, we extend the original set of propositions \fluents with a new set of $\Sticky$ propositions that we denote $\fluents'$. Each goal $\subgoal_{i} \in \goalhat$ is associated with a $\Sticky$ version, which tracks whether $\subgoal_{i}$ has been achieved and should remain true. Formally, we define $\fluents' = \bigcup_{\subgoal_{i} \in \goalhat} \{\subgoal_i\sticky\}$. 

We also extend the original actions to ensure the persistence of the $\Sticky$ propositions once they become true.
To do that, we differentiate four different groups within \actions based on their interaction with the goals.

\noindent\textbf{ (1) Actions that add and do not delete goals:} $\actionsaddgoals = \{\action \in \actions \mid \goalhat \cap \addeffects(\action)\not = \emptyset \land \goalhat\cap\deleffects(\action) = \emptyset\}$. Since actions can achieve multiple goal propositions, for each action $\actionaddgoals\in\actionsaddgoals$, we define $\commitgoals$ as a set containing all possible combinations of $\Sticky$ goals in $\addeffects(\actionaddgoals)\cap\goalhat$.
This set will have $2^n$ subsets, where $n$ is the number of goals in $\addeffects(\actionaddgoals)\cap\goalhat$. For instance, if $\addeffects(\actionaddgoals)\cap\goalhat = \{ x, y\}$, then $\commitgoals = \{ \{\}, \{x\}, \{ y\}, \{x,y\} \}$.
Next, we introduce a new set of $\Sticky$ actions $\actionssticky$: for each action $\actionaddgoals \in \actionsaddgoals$, we introduce $|\commitgoals|$ new actions. Each $\Sticky$ action $\actioncommit_{i}\in\actionssticky$, with $i \in \commitgoals$, is defined as: 
\begin{itemize}
\small
    \item $\name(\actioncommit_{i}) = \name(\actionaddgoals)\sticky\mbox{-}i$
    \item $\precondition^+(\actioncommit_{i})\hspace{-1.5mm}~=\hspace{-1.5mm}~\precondition^+(\actionaddgoals)$,
    \item $\precondition^-(\actioncommit_{i})\hspace{-1.5mm}~=\hspace{-1.5mm}~\precondition^-(\actionaddgoals)~\cup\{ \cup_{j \in i}  \subgoal_j\sticky \}$
    \item $\deleffects(\actioncommit_{i})\hspace{-1mm}=\hspace{-1mm} \deleffects(\actionaddgoals), \addeffects(\actioncommit_{i})\hspace{-1.5mm}~=~\hspace{-1.5mm}\addeffects(\actionaddgoals)~\cup~\{ \cup_{j \in i} \subgoal_j\sticky\}$

    \item $\cost(\actioncommit_{i}) = \cost(\actionaddgoals)$
\end{itemize}

\noindent\textbf{ (2) Actions that do not add but delete goals:} $\actionsdeletegoals = \{\action \in \actions \mid \goalhat \cap \addeffects(\action) = \emptyset \land\goalhat \cap \deleffects(\action)\not = \emptyset\}$. We must ensure these actions do not delete a goal that is already $\Sticky$. 
For each $\actiondeletegoals\in\actionsdeletegoals$, we substitute it with an updated $\Forcesticky$ action $\actionnoncommit$, which is added to a new set $\actionsnoncommit$ and defined as follows:
\begin{itemize}
    \small
        \item $\name(\actionnoncommit) = \name(\actiondeletegoals)\forcesticky$
        \item $\precondition^+(\actionnoncommit) = \precondition^+(\actiondeletegoals)$
        \item $\precondition^-(\actionnoncommit) \!=\! \precondition^-(\actiondeletegoals) \cup \{\subgoal_j\sticky \!\mid\! \subgoal_j \!\in \! \deleffects(\actiondeletegoals)\cap\goalhat\}$ 
        \item $\deleffects(\actionnoncommit) = \deleffects(\actiondeletegoals), \addeffects(\actionnoncommit) = \addeffects(\actiondeletegoals)$
        \item $\cost(\actionnoncommit) = \cost(\actiondeletegoals)$
\end{itemize}

\noindent\textbf{ (3) Actions that add and delete goals:} $\actionsdeladdgoals =  \{\action \in \actions \mid \goalhat \cap \addeffects(\action) \not= \emptyset \land \goalhat\cap\deleffects(\action) \not= \emptyset\}$. The behavior of the reformulated actions is defined as a combination of the previous two type of actions, and therefore they: (i) add $\Sticky$ propositions to achieve the goals in $\goalhat$; and (ii) ensure that these actions do not delete a goal that is already $\Sticky$.
For each $\actiondeladdgoals\in\actionsdeladdgoals$, we define $\commitdeladdgoals$ as the set of all the possible combinations of $\Sticky$ goals in $\addeffects(\actiondeladdgoals)\cap\goalhat$.
Next, we introduce a new set of simultaneous actions $\actionssim$: for each $\actiondeladdgoals\in\actionsdeladdgoals$, we introduce $|\commitdeladdgoals|$ new actions. Each action $\actionsim_{j}\in\actionssim$, with $j\in\commitdeladdgoals$, is defined~as: 
\begin{itemize}
\small
    \item $\name(\actionsim_{j}) = \name(\actiondeladdgoals)\simultaneous\mbox{-}j$
    \item $\precondition^+(\actionsim_{j}) = \precondition^+(\actiondeladdgoals)$
    \item $\precondition^-(\actionsim_{j})~\hspace{-1.5mm}=~\hspace{-1.5mm}\precondition^-(\actiondeladdgoals)~\hspace{-1.5mm}~\cup~\hspace{-1mm}~\{ \cup_{i \in j}  \subgoal_i\sticky\}~\hspace{-1mm}~\cup~\hspace{-1mm}~\{ \subgoal_i\sticky\mid\subgoal_i \in \deleffects(\actiondeladdgoals)\cap\goalhat\}$
    \item $\deleffects(\actionsim_{j})~\hspace{-2mm}=~\hspace{-2mm}\deleffects(\actiondeladdgoals)$, $\addeffects(\actionsim_{j})~\hspace{-2mm}=~\hspace{-2mm}\addeffects(\actiondeladdgoals)~\hspace{-1.2mm}\cup~\hspace{-1.2mm}\{\cup_{i \in j}  \subgoal_i\sticky\}$,
    \item $\cost(\actionsim_{j}) = \cost(\actiondeladdgoals)$
\end{itemize}

\noindent\textbf{ (4) Actions that neither add nor delete goals:} $\actionslonely= \{\action \in \actions \mid \goalhat \cap \addeffects(\action) = \emptyset \land  \goalhat \cap \deleffects(\action) = \emptyset\}$. We keep these actions unchanged.

The new planning task is defined as follows:
\begin{definition}
    Given a planning task \stripstask, a \textbf{commit planning task} is defined as a tuple $\stickytask = \langle\commitfluents, \commitactions, \init, \commitgoal\rangle$
    where,
    \begin{itemize}
    \small
        \item $\commitfluents = \fluents \cup \fluents'$, \hspace{-5mm}
        \item $\commitactions = \actionssticky \cup \actionsnoncommit \cup \actionssim \cup \actionslonely$
        \item $\commitgoal = \{\goal\setminus\goalhat\} \cup_{\subgoal_i \in \goalhat}\{\subgoal_i\sticky\}$
    \end{itemize}
\end{definition}

\begin{theorem}\label{thm:completeness}
    If \task is solvable, \stickytask is also solvable.
\end{theorem}

\begin{proof}
    Let \plan be a solution for \task. We need to show that there exists a plan $\plan'$ that solves $\stickytask$ comprised by the same sequence of actions as \plan, only varying their version, i.e. we include actions that allow us to have $\Sticky$ and $\Forcesticky$ versions. For each action $\action\in\plan$, there could be these cases:

    Case 1. $\action\in\actionslonely$, then $\plan'$ will use $\action$.
    
    Case 2. $\action\in\actionsdeletegoals$ is replaced by its corresponding $\Forcesticky$ action $\action'\in\actionsnoncommit$. They differ only in their preconditions related to $\Sticky$ goals. If such action appears in $\plan$, it means there will be another action $\action'' \in \plan$ that will achieve the goal that is being temporary deleted.

    Case 3. $\action\in\actionsaddgoals$. We differentiate two sub-cases: (i) $\action$ is not the last action in the plan achieving the goals in $\addeffects(\action) \cap \goalhat$: this action is replaced by the original non-commit action that we keep in $\actionssticky$; and (ii) $\action$ is the last action in the plan achieving the goals in $\addeffects(\action) \cap \goalhat$: this action is replaced by its corresponding $\actioncommit_{i}\in\actionssticky$ where $i\in\commitgoals$.

    Case 4. $\action\in\actionsdeladdgoals$. %
    We differentiate two sub-cases as Case 3 and combine the preconditions related to $\Sticky$ goals as Case 2.

    For every $\subgoal_i\in\goal$,  either $\subgoal_i\in\init$, and is correctly achieved in $\stickytask$ by the same original actions in $\commitactions$; or $\subgoal_i\in\goalhat$  corresponds to a $\Sticky$ version $\subgoal_i\sticky\in\commitgoal$ and is correctly achieved by $\Sticky$ actions in $\actionssticky$.
\end{proof}

\begin{theorem}\label{thm:soundness}
    Any plan $\plan'$ that solves \stickytask is a plan for the original problem \task.
\end{theorem}

\begin{proof}
    Let $\plan'$ be a solution for $\stickytask$. We need to show that there is a plan $\plan$ that can be mapped from $\plan'$ which is a solution for $\task$.
    Each action in $\plan'$ falls into one of four cases: (i) $\action\in\actionslonely$: we leave them as is in $\plan$; (ii) $\action\in\actionssticky$: they are substituted in $\plan$ by their counterparts in $\actionsaddgoals$; (iii) $\action\in\actionsnoncommit$: they are substituted in $\plan$ by their counterparts in $\actionsdeletegoals$, which are equivalent to the original actions in \task, regarding their effects. Their preconditions in $\stickytask$ include $\Sticky$ attributes, but this only restricts action applicability in $\stickytask$, not in \task; (iv) $\action \in \actionssim$: they are substituted in $\plan$ by their counterparts in $\actionsdeladdgoals$. Their preconditions in $\stickytask$ only restricts action applicability in $\stickytask$ as in (iii).
    Lastly, for all $\subgoal_i' \in \commitgoal$, we have that either $\subgoal_i' \in \goal$ or $\subgoal_i'$ is the $\Sticky$ version of a $\subgoal_i \in \goal$. Then $\goal$ is correctly achieved in $\task$.
\end{proof}

From Theorem~\ref{thm:completeness} and~\ref{thm:soundness}, along with the preservation of costs from the original actions to the commit ones, we can show that the optimality of plans solving $\task$ is maintained in the commit task $\stickytask$.

\begin{corollary}
     If a plan $\plan$ optimally solves \task, then there exists a plan $\plan'$ that optimally solves $\stickytask$.
\end{corollary}
\begin{proof}
Let \plan be an optimal solution for \task. We need to show that there is a plan $\plan'$ that can be mapped from \plan which is an optimal solution for \stickytask.

From Theorem~1, given a solution $\plan$ for $\task$, then there exists a solution $\plan'$ for $\stickytask$. Since the cost of every action $\action'$ in $\commitactions$ is preserved from the cost of each action $\action\in\actions$, then we have that $\cost(\plan)=\cost(\plan')$. Let us now observe that if $\plan$ is optimal, then $\plan'$ is also optimal. Note that $\plan$ being optimal means that for all alternative plans $\alternativeplan \in \alternativeplans{\plan}$, we have that $\cost(\plan) \le \cost(\alternativeplan)$. 
Therefore, in order to show that $\plan'$ is optimal, we need to show that it does not exist $\plan''$ that solves $\stickytask$ and $\cost(\plan'')~<~\cost(\plan')$. 

By contradiction, let us suppose that there exists $\plan''$ that solves \stickytask such that $\plan''\neq\plan'$ and $\plan''$ is an optimal plan, i.e., $\cost(\plan'')<\cost(\plan')$. By Theorem~2, there exists $\plan'''$ that can be mapped from $\plan''$ and is a solution for $\task$. Since the costs are preserved, we have $\cost(\plan''') = \cost(\plan'')$. And giving $\cost(\plan''')=\cost(\plan'')< \cost(\plan')=\cost(\plan)$, then we have that $\plan$ is not optimal.  This is a contradiction from the assumption that $\plan'$ is not optimal. Thus, $\plan'$ is an optimal solution for \stickytask.
\end{proof}

Let us illustrate the compilation by considering a planning task \stripstask where $\fluents=\{x,y\}$, $\actions=\{\action_1, \action_2\}$, $\init=\{\}$, and $\goal=\{x,y\}$.
Actions are defined as follows: $(\action_1)
\precondition^+=\precondition^-=\{\},  \deleffects=\{\}, \addeffects=\{x\}$; and
$(\action_2)\precondition^+=\{x\}, \precondition^-=\{\},  \deleffects=\{x\}, \addeffects=\{y\}$.
The optimal plan for \task is $\plan = (\action_1, \action_2,\action_1)$.
Note that $x$ becomes true after the first execution of $\action_1$, but only permanently after its last execution, as $\action_2$ must temporarily delete it to achieve $y$.
Given that $\action_1\in\actionsaddgoals$, our compilation keeps the original action $\action_1$, and adds a commit version $(\action_1\sticky\mbox{-x})$ in $\actionssticky$ such that $\precondition^+=\{\},\precondition^-=\{x\sticky\},\deleffects=\{\}, \addeffects=\{x, x\sticky\}$.
On the other hand, given that $\action_2\in\actionsdeladdgoals$,  the following simultaneous actions are created in $\actionssim$:
($\action_2\simultaneous\mbox{-}\emptyset$) $\precondition^+=\{x\}, \precondition^-=\{x\sticky\},$
    $ \deleffects=\{x\}, \addeffects = \{y\}$; and
($\action_2\simultaneous\mbox{-y}$) $\precondition^+=\{x\}, \precondition^-=\{x\sticky,$ $ y\sticky\},$
    $ \deleffects=\{x\}, \addeffects = \{y, y\sticky\}$.
This gives us the following commit planning task, $\stickytask=\langle\commitfluents, \commitactions, \init, \commitgoal\rangle$ where $\commitactions = \{\action_1\sticky\mbox{-}\emptyset, \action_1\sticky\mbox{-x},$ $ \action_2\simultaneous\mbox{-}\emptyset,\action_2\simultaneous\mbox{-y}\}$ and $\commitgoal = \{x\sticky, y\sticky\}$. Then, there exists an optimal plan  in $\stickytask$ that can be mapped from $\plan$, $\plan_c = (\action_1\sticky\mbox{-}\emptyset, \action_2\simultaneous\mbox{-y}, \action_1\sticky\mbox{-x})$.
In $\plan_c$, unlike in \plan, we can determine during the search if an action permanently achieves a goal without post-processing.
\section{Evaluation}

\begin{table}[]
\footnotesize
\setlength{\tabcolsep}{0.5pt}
\setlength\extrarowheight{-2pt}
    \centering
    \begin{tabular}{l||r|r||r|r||r|r||}
         & \multicolumn{2}{c||}{Coverage} & \multicolumn{2}{c||}{SAT Score} & \multicolumn{2}{c||}{AGL Score} \\ \hline
         Domain (\#Tasks) & $\task$ & $\stickytask$ & $\task$ & $\stickytask$ & $\task$ & $\stickytask$ \\ \hline
         \rowcolor{gray!20} agricola18 (20)& \textbf{20} & \textbf{20}& \textbf{20.0} & \textbf{20.0}& \textbf{15.5} & \textbf{15.5} \\
\rowcolor{gray!20} airport (50)& \textbf{34} & \textbf{34}& \textbf{34.0} & 33.9& \textbf{33.0} & 32.9 \\
barman11 (20)& \textbf{20} & \textbf{20}& \textbf{20.0} & 19.9& 25.8 & \textbf{25.9} \\
barman14 (14)& \textbf{14} & \textbf{14}& \textbf{14.0} & 13.9& \textbf{17.2} & 17.1 \\
blocks (35)& \textbf{35} & \textbf{35}& 22.7 & \textbf{35.0}& 56.4 & \textbf{56.8} \\
\rowcolor{gray!20} childsnack14 (20)& \textbf{15} & \textbf{15}& \textbf{15.0} & \textbf{15.0}& \textbf{20.0} & 19.9 \\
\rowcolor{gray!20}  data-net18 (20)& \textbf{20} & \textbf{20}& \textbf{20.0} & \textbf{20.0}& \textbf{29.6} & 29.5 \\
depot (22)& 20 & \textbf{21}& 18.8 & \textbf{20.1}& 21.3 & \textbf{22.4} \\
driverlog (20)& \textbf{20} & \textbf{20}& 17.8 & \textbf{19.5}& 26.7 & \textbf{30.7} \\
elevators08 (30)& \textbf{30} & \textbf{30}& 28.3 & \textbf{28.4}& \textbf{48.1} & 47.9 \\
elevators11 (20)& \textbf{20} & \textbf{20}& 18.8 & \textbf{19.2}& \textbf{32.4} & 32.3 \\
\rowcolor{gray!20}  floortile11 (20)& \textbf{9} & \textbf{9}& \textbf{8.7} & 8.6& \textbf{6.1} & 5.8 \\
\rowcolor{gray!20}  floortile14 (20)& \textbf{9} & 8& \textbf{8.7} & 7.1& \textbf{2.9} & 2.5 \\
\rowcolor{gray!20} freecell (80)& 77 & \textbf{79}& 74.7 & \textbf{75.4}& 81.0 & \textbf{83.1} \\
ged14 (20)& \textbf{20} & \textbf{20}& \textbf{20.0} & 19.6& \textbf{32.2} & 31.9 \\
grid (5)& \textbf{5} & \textbf{5}& 4.9 & \textbf{5.0}& \textbf{6.5} & \textbf{6.5} \\
gripper (20)& \textbf{20} & \textbf{20}& \textbf{20.0} & \textbf{20.0}& \textbf{32.9} & 32.3 \\
\rowcolor{gray!20} hiking14 (20)& \textbf{20} & \textbf{20}& \textbf{20.0} & \textbf{20.0}& \textbf{24.0} & \textbf{24.0} \\
logistics00 (28)& \textbf{28} & \textbf{28}& 27.6 & \textbf{27.7}& \textbf{46.4} & 46.2 \\
logistics98 (35)& \textbf{35} & 34& \textbf{34.9} & 33.7& \textbf{41.0} & 40.9 \\
\rowcolor{gray!20} miconic (150)& \textbf{150} & \textbf{150}& \textbf{150.0} & \textbf{150.0}& \textbf{231.6} & 231.3 \\
\rowcolor{gray!20} movie (30)& \textbf{30} & \textbf{30}& 26.2 & \textbf{30.0}& 54.5 & \textbf{54.6} \\
mprime (35)& \textbf{35} & \textbf{35}& \textbf{35.0} & \textbf{35.0}& \textbf{47.2} & \textbf{47.2} \\
\rowcolor{gray!20} mystery (30)& \textbf{19} & \textbf{19}& 18.9 & \textbf{19.0}& \textbf{25.7} & 25.6 \\
nomystery11 (20)& \textbf{16} & \textbf{16}& \textbf{16.0} & \textbf{16.0}& \textbf{22.5} & 22.3 \\
\rowcolor{gray!20} openstacks08 (30)& \textbf{30} & \textbf{30}& 20.6 & \textbf{30.0}& \textbf{45.9} & 43.9 \\
\rowcolor{gray!20} openstacks11 (20)& \textbf{20} & \textbf{20}& 13.1 & \textbf{20.0}& \textbf{30.5} & 29.0 \\
\rowcolor{gray!20} openstacks14 (20)& \textbf{20} & \textbf{20}& 12.2 & \textbf{20.0}& \textbf{26.4} & 24.5 \\
\rowcolor{gray!20} openstacks (30)& \textbf{30} & \textbf{30}& 29.8 & \textbf{30.0}& \textbf{36.0} & 34.8 \\
\rowcolor{gray!20} org-syn18 (20)& \textbf{7} & 6& \textbf{7.0} & 6.0& \textbf{9.2} & 5.3 \\
\rowcolor{gray!20} org-syn-spl18 (20)& \textbf{18} & 11& \textbf{18.0} & 11.0& \textbf{17.7} & 12.0 \\
\rowcolor{gray!20} parcprinter08 (30)& \textbf{30} & 21& \textbf{28.6} & 21.0& \textbf{47.0} & 27.4 \\
parcprinter11 (20)& \textbf{20} & 15& \textbf{19.0} & 15.0& \textbf{31.5} & 17.7 \\
parking11 (20)& \textbf{20} & 18& \textbf{18.9} & 15.6& \textbf{20.2} & 12.3 \\
parking14 (20)& \textbf{20} & 18& \textbf{19.8} & 15.3& \textbf{21.0} & 13.8 \\
\rowcolor{gray!20} pathways (30)& \textbf{23} & \textbf{23}& \textbf{23.0} & 22.9& \textbf{32.3} & \textbf{32.3} \\
pegsol08 (30)& \textbf{30} & 29& 26.2 & \textbf{28.5}& \textbf{43.2} & 34.1 \\
pegsol11 (20)& \textbf{20} & 19& 18.1 & \textbf{18.7}& \textbf{26.6} & 19.7 \\
pipes-notank (50)& \textbf{43} & \textbf{43}& 33.1 & \textbf{40.6}& 47.4 & \textbf{53.8} \\
pipes-tank (50)& \textbf{43} & 37& \textbf{35.6} & 35.4& \textbf{41.4} & 39.6 \\
\rowcolor{gray!20} psr-small (50)& \textbf{50} & \textbf{50}& \textbf{50.0} & 49.1& \textbf{87.0} & 86.0 \\
\rowcolor{gray!20} rovers (40)& \textbf{40} & \textbf{40}& \textbf{40.0} & \textbf{40.0}& \textbf{54.1} & \textbf{54.1} \\
\rowcolor{gray!20} satellite (36)& \textbf{36} & \textbf{36}& 33.4 & \textbf{35.9}& 41.7 & \textbf{42.8} \\
\rowcolor{gray!20} scanalyzer08 (30)& \textbf{30} & \textbf{30}& \textbf{27.9} & 27.4& \textbf{36.8} & 34.1 \\
\rowcolor{gray!20} scanalyzer11 (20)& \textbf{20} & \textbf{20}& \textbf{18.8} & 17.9& \textbf{25.0} & 23.5 \\
sokoban08 (30)& \textbf{29} & \textbf{29}& 26.4 & \textbf{27.0}& \textbf{32.4} & 30.4 \\
sokoban11 (20)& \textbf{20} & \textbf{20}& \textbf{18.7} & 18.5& \textbf{22.8} & 21.6 \\
spider18 (20)& \textbf{19} & 16& \textbf{17.8} & 15.4& \textbf{14.3} & 8.8 \\
storage (30)& 19 & \textbf{21}& 16.9 & \textbf{19.9}& 26.5 & \textbf{27.7} \\
tetris14 (17)& \textbf{17} & 10& \textbf{15.3} & 9.3& \textbf{13.1} & 9.1 \\
\rowcolor{gray!20} tidybot11 (20)& \textbf{18} & \textbf{18}& 17.8 & \textbf{17.9}& 15.5 & \textbf{15.7} \\
\rowcolor{gray!20} tidybot14 (20)& \textbf{19} & \textbf{19}& \textbf{19.0} & 18.8& \textbf{13.5} & 13.3 \\
\rowcolor{gray!20} tpp (30)& \textbf{30} & \textbf{30}& \textbf{30.0} & \textbf{30.0}& \textbf{38.4} & 38.0 \\
transport08 (30)& \textbf{30} & \textbf{30}& \textbf{29.9} & 29.7& \textbf{42.8} & 42.6 \\
transport11 (20)& \textbf{20} & \textbf{20}& \textbf{20.0} & 19.9& \textbf{29.7} & 29.5 \\
transport14 (20)& \textbf{20} & \textbf{20}& 19.9 & \textbf{20.0}& \textbf{28.1} & 28.0 \\
\rowcolor{gray!20} trucks (30)& 15 & \textbf{17}& 15.0 & \textbf{17.0}& 16.1 & \textbf{17.8} \\
\rowcolor{gray!20} visitall11 (20)& \textbf{20} & \textbf{20}& \textbf{20.0} & 16.4& \textbf{33.1} & 32.4 \\
\rowcolor{gray!20} visitall14 (20)& \textbf{20} & \textbf{20}& \textbf{20.0} & 16.5& \textbf{29.1} & 27.9 \\
woodwork08 (30)& \textbf{30} & \textbf{30}& \textbf{29.9} & 27.6& \textbf{40.9} & 36.5 \\
woodwork11 (20)& \textbf{20} & \textbf{20}& \textbf{19.9} & 18.7& \textbf{26.8} & 24.0 \\
zenotravel (20)& \textbf{20} & \textbf{20}& \textbf{19.9} & 19.3& \textbf{29.7} & 29.6 \\ \hline
Total (1767)& \textbf{1637} & 1598& 1544.3 & \textbf{1554.3}& \textbf{2154.1} & 2058.8 \\
    \end{tabular}
    \caption{\lamafirst Coverage, SAT and AGL scores when solving the original ($\task$) and compiled ($\stickytask$) tasks. Bold figures indicate best performance.}
    \label{tab:ipc_table_lamafirst}
\end{table}
\begin{table}[]
\footnotesize
\setlength{\tabcolsep}{0.5pt}
\setlength\extrarowheight{-2pt}
    \centering
    \begin{tabular}{l||r|r||r|r||r|r||}
         & \multicolumn{2}{c||}{Coverage} & \multicolumn{2}{c||}{SAT Score} & \multicolumn{2}{c||}{AGL Score} \\ \hline
         Domain (\#Tasks) & $\task$ & $\stickytask$ & $\task$ & $\stickytask$ & $\task$ & $\stickytask$ \\ \hline

\rowcolor{gray!20} airport (50)& \textbf{28} & \textbf{28}& \textbf{28.0} & \textbf{28.0}& 33.3 & \textbf{33.4} \\

blocks (35)& \textbf{28} & \textbf{28}& \textbf{28.0} & \textbf{28.0}& 35.8 & \textbf{36.6} \\


elevators11 (20)& \textbf{18} & \textbf{18}& \textbf{18.0} & \textbf{18.0}& \textbf{10.7} & 10.3 \\

\rowcolor{gray!20} floortile14 (20)& \textbf{6} & 5& \textbf{6.0} & 5.0& \textbf{1.9} & 1.6 \\

gripper (20)& \textbf{7} & 5& \textbf{7.0} & 5.0& \textbf{7.1} & 5.5 \\

logistics00 (28)& \textbf{20} & \textbf{20}& \textbf{20.0} & \textbf{20.0}& \textbf{23.7} & 23.2 \\

\rowcolor{gray!20} openstacks08 (30)& \textbf{21} & 20& \textbf{21.0} & 20.0& \textbf{17.7} & 16.0 \\

\rowcolor{gray!20} org-syn-spl18 (20)& \textbf{15} & 11& \textbf{15.0} & 11.0& \textbf{13.4} & 10.3 \\

pegsol11 (20)& \textbf{18} & 11& \textbf{18.0} & 11.0& \textbf{15.8} & 3.3 \\

rovers (40)& \textbf{7} & \textbf{7}& \textbf{7.0} & \textbf{7.0}& \textbf{9.9} & 9.8 \\

sokoban11 (20)& \textbf{20} & \textbf{20}& \textbf{20.0} & \textbf{20.0}& \textbf{18.3} & 17.0 \\
tetris14 (17)& \textbf{6} & 3& \textbf{6.0} & 3.0& \textbf{3.3} & 2.2 \\

\rowcolor{gray!20} trucks (30)& \textbf{10} & \textbf{10}& \textbf{10.0} & \textbf{10.0}& \textbf{9.0} & 8.8 \\

woodwork11 (20)& \textbf{12} & 11& \textbf{12.0} & 11.0& \textbf{10.3} & 8.8 \\
\hline
Total (1767)& \textbf{917} & 887& \textbf{917.0} & 887.0& \textbf{1006.2} & 943.2 \\
    \end{tabular}
    \caption{\lmcut Coverage, SAT and AGL scores when solving the original ($\task$) and compiled ($\stickytask$) tasks.}
    \label{tab:ipc_table_lmcut}
\end{table}

\paragraph{Experimental Setting.} We selected all the \textsc{strips} tasks from the optimal suite of the Fast Downward~\cite{helmert2006fast} benchmark collection\footnote{https://github.com/aibasel/downward-benchmarks}.
This gives us $1767$ tasks divided across $62$ domains.
We reformulated each task using our approach and solved both the original (\task) and the compiled (\stickytask) tasks using two planners.
The first one, \lamafirst,  is a state-of-the-art planner that won the agile track of the last International Planning Competition (IPC)\footnote{https://ipc2023-classical.github.io/}.
\lamafirst runs the first iteration of the well-known \textsc{lama} planner~\cite{richter2010lama}, aiming to find a solution as soon as possible, disregarding its quality.
It combines, among others: deferred heuristic evaluation,  preferred operators, and multiple open lists guided by the \textsc{ff}~\cite{hoffmann2001ff} and \textsc{landmark-sum}~\cite{hoffmann2004ordered} heuristics.
The second one, \lmcut, is an optimal planner that runs $A^*$ with the admissible \textsc{lmcut} heuristic.
We run both planners with a 8GB memory bound and a time limit of $900$s on Intel Xeon E5-2666 v3 CPUs @ 2.90GHz. 
We will report the following metrics, where higher numbers indicate better performance.

\underline{Coverage}: number of problems solved by the planner.

\underline{SAT Score}: $C^*/C$, where $C^*$ is the lowest cost found for the task, and $C$ is the plan's cost the planner gets in the task. 

\underline{AGL Score}: $1 - log(T)/log(900)$, where $T$ is the planner's runtime, including grounding and search time.  

\paragraph{Results.}
Tables~\ref{tab:ipc_table_lamafirst} and \ref{tab:ipc_table_lmcut} show the Coverage, SAT and AGL scores when solving the original (\task) and the compiled (\stickytask) tasks with \lamafirst and \lmcut, respectively.
Due to space constraints, we present only a subset of the results for \lmcut.
Domains with problems where goal propositions cannot be falsified once achieved are highlighted in gray.
These tasks do not require reformulation into $\Sticky$ planning tasks, but we aimed to verify the general applicability of our compilation across different problem structures.
As observed, both planners achieve slightly better overall results when solving \task, attaining higher scores across most metrics compared to \stickytask. 
The performance gap varies by domain and planner, yet \stickytask tasks can always be solved comparably to \task, both optimally and suboptimally. 
This holds true even in domains such as \textsc{openstacks} or \textsc{rovers}, where our $\Sticky$ compilation is unnecessary due to the permanence of goals. 
This emphasizes the non-intrusive nature of our compilation, providing the advantage of reasoning about goal achievement during the planning phase without compromising search efficiency.
\section{Conclusions and Future Work}
We introduced a compilation that extends the original planning task with $\Sticky$ actions that allow planners to enforce the persistence of specific goals once achieved, thus committing to them in the search sub-tree.
This new task structure allows us to determine at planning time when an action permanently achieves a goal.
This opens up research opportunities for reasoning and optimizing goal achievement distribution, potentially benefiting various planning applications.

Certain domains, such as \textsc{blocks}, consistently benefit from our reformulation, indicating that allowing planners to greedily commit to achieving specific goals may enhance search performance in certain situations. 
In other domains such as \textsc{openstacks}, performance improvements appear to stem from an internal state representation that is more compatible with the specific planner being used~\cite{vallati2015effective}.
We intend to further investigate this phenomenon and its implications as part of our future work.

\section*{Disclaimer}
This paper was prepared for informational purposes by the Artificial Intelligence Research group of JPMorgan Chase \& Co. and its affiliates ("JP Morgan'') and is not a product of the Research Department of JP Morgan. JP Morgan makes no representation and warranty whatsoever and disclaims all liability, for the completeness, accuracy or reliability of the information contained herein. This document is not intended as investment research or investment advice, or a recommendation, offer or solicitation for the purchase or sale of any security, financial instrument, financial product or service, or to be used in any way for evaluating the merits of participating in any transaction, and shall not constitute a solicitation under any jurisdiction or to any person, if such solicitation under such jurisdiction or to such person would be unlawful.

\bibliographystyle{kr}
\bibliography{references}

\end{document}